\PassOptionsToPackage{table}{xcolor}
\documentclass[11pt]{article}
\usepackage{booktabs,multirow,graphicx,float}
\usepackage{agenticlearning-pdflatex}

\usepackage{placeins}
\usepackage{wrapfig}
\usepackage{array}

\newcommand{\ours}[1]{\textbf{#1}}    
\newcommand{\refarm}[1]{\textit{#1}}  

\graphicspath{{img/}}
\definecolor{AwardOrange}{HTML}{C2410C}
\definecolor{FigBlue}{HTML}{2F6699}   

\newcommand{\takeawaybox}[3]{%
  \par\smallskip
  \begingroup
  \setlength{\fboxrule}{0.8pt}%
  \setlength{\fboxsep}{0pt}%
  \noindent\fcolorbox{FigBlue!75}{FigBlue!4}{%
    \begin{minipage}{\dimexpr\linewidth-2\fboxrule\relax}
      \colorbox{FigBlue}{%
        \hspace{5pt}\parbox{\dimexpr\linewidth-10pt\relax}{%
          \vspace{2pt}\textcolor{white}{\bfseries Takeaway #1: #2}\vspace{2pt}}%
        \hspace{5pt}}%
      \par\noindent\hspace{6pt}\parbox{\dimexpr\linewidth-12pt\relax}{%
        \vspace{3pt}\small #3\vspace{3pt}}\hspace{6pt}
    \end{minipage}}%
  \endgroup
  \par\smallskip
}
\hypersetup{
  colorlinks=true,
  linkcolor=AwardOrange,
  citecolor=LabBlue,
  urlcolor=LabBlue
}

\title{DeepJEPA: Scaling World Models from Within}
\author{
Zijian Jin\textsuperscript{1,\ensuremath{\dagger}},
Yunbei Zhang\textsuperscript{2,\ensuremath{\dagger}},
Yuanzhe Liu\textsuperscript{3},
Ming Liu\textsuperscript{4},
Baian Chen\textsuperscript{5},
Weirui Ye\textsuperscript{5},
Shilong Liu\textsuperscript{4,6},
Marco Pavone\textsuperscript{7}
\\[3pt]
\footnotesize
\textsuperscript{1}New York University \enspace
\textsuperscript{2}Tulane University \enspace
\textsuperscript{3}UIUC
\enspace
\textsuperscript{4}Princeton University \enspace
\textsuperscript{5}MIT \enspace
\textsuperscript{6}Columbia University \enspace
\textsuperscript{7}Stanford University
\\[2pt]
\textsuperscript{\ensuremath{\dagger}}Equal contribution
\\[5pt]
\textbf{Project website:} \url{https://deepjepa.github.io/}
\footnotesize
}
\date{}

\begin{document}
\shorttitle{DeepJEPA}
\shortauthor{Jin \& Zhang, et al.}
\maketitle

\begin{abstract}
World-model planners typically scale outward by rolling farther, sampling more
trajectories, or optimizing longer, while assigning the same computation to
every imagined transition. We show that making every transition uniformly
deeper wastes computation and can degrade planning because useful refinement is
concentrated at a small set of decision-critical events. We introduce
\textbf{DeepJEPA}, a weight-tied joint-embedding predictive world model that
treats transition depth as an inner test-time scaling axis and learns when
another recurrent update is worth computing for each candidate and rollout
step. Across five visual-control
settings, DeepJEPA improves or matches the strongest fixed-depth planner while
averaging only $1.00$--$1.26$ updates per transition. Its additional computation
concentrates at contact onset and sustained object interaction, where latent
corrections can change which candidates enter the planner's elite set and which
action is selected. Representation probes further show that improved planning
does not require uniformly better object-state decodability. DeepJEPA therefore
reframes world-model scaling as a problem of allocating internal computation
where it can change the planner's decision: \emph{think deeper at
decision-critical transitions instead of making every rollout uniformly deeper
or longer}.
\end{abstract}

\begin{figure}[H]
  \centering
  \includegraphics[width=\textwidth]{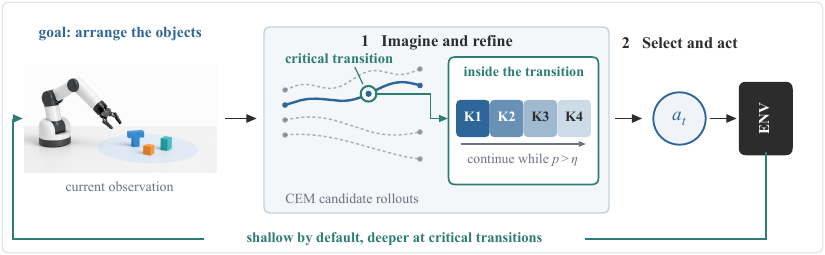}
  \caption{\textbf{Think before you roll.} Conventional planners scale the
  outer search while assigning every imagined transition the same predictor
  depth. DeepJEPA scales from within by keeping most transitions shallow and
  refining only decision-critical ones before the planner selects an action.
  The task, rollout horizon, candidate budget, and outer planning loop remain
  unchanged.}
  \label{fig:teaser}
\end{figure}

\section{Introduction}

Joint-embedding self-supervision learns representations without reconstructing
pixels, and JEPA turns prediction in representation space into a route toward
predictive world models~\citep{chen2020simsiam,
bardes2022vicregvarianceinvariancecovarianceregularizationselfsupervised,jepa,bardes2024revisiting,balestriero2025lejepa}. Action-conditioned
JEPAs turn this principle into a world model: a planner rolls candidate actions
through latent space and selects those whose predicted terminal state best
matches a visual goal. V-JEPA~2-AC and LeWorldModel (LeWM) show that such latent
imagination can drive model-predictive control (MPC) with the Cross-Entropy
Method (CEM)~\citep{assran2025v,maes2026leworldmodel,
rubinstein2004cross}. Related latent planners study pretrained visual features,
reward-free offline data, human video, temporal geometry, and hierarchical
rollouts~\citep{dinowm,sobal2025learning,dexwm,
wang2026temporal_straightening,zhang2026hierarchical,toso2026learning}.

Test-time scaling for world models has largely meant scaling \emph{outward}:
increase the number of candidate trajectories, extend the horizon, or run more
search iterations. Yet the transition model inside every rollout is ordinarily
treated as a fixed-cost primitive. Every candidate, time step, and physical
regime receives the same predictor depth. This hidden uniformity is a poor match
to interaction. Free-space motion is often locally smooth, whereas contact and
multi-object coupling can abruptly change which action is preferable. A planner
therefore does not need uniformly deeper predictions. It needs to scale
\emph{within} the transition, spending additional computation exactly where a
latent update can change a decision.

We call this missing axis \emph{transition-level deliberation}: recurrent
computation performed inside one imagined step before the temporal rollout
continues. A fixed-depth sweep reveals the central contradiction. Although
depth adds computation, $K=4$ never improves the aggregate mean over $K=1$
across five settings, and intermediate depths are non-monotonic. More inner
compute is therefore not a scaling law by itself. Allocation is the scaling
problem.

Prediction is not the final product of a world-model planner. The executed
action is~\citep{kong2026driving}. An additional recurrent update is therefore an internal action whose
value lies in whether it improves the external action selected after planning.
This gives inner scaling a decision-theoretic interpretation grounded in
rational metareasoning~\citep{russell1991principles,hay2014selecting}:
computation should be allocated by its expected value to the planner, not by
depth itself.
DeepJEPA amortizes a reward-free surrogate of this value at the granularity of
one candidate and one imagined transition.

This observation leads to \textbf{DeepJEPA} (Fig.~\ref{fig:teaser}). DeepJEPA
turns a JEPA transition from an atomic prediction into a budgeted computation
process. A weight-tied cell produces successive latent updates, while a learned
continue head decides independently for each candidate--time pair whether to
compute again. Supervision comes from the marginal reduction in latent
prediction error during training, but at inference the decision depends only on
the current recurrent state. DeepJEPA therefore scales the world model from
within CEM without observing a future frame or modifying the candidate budget,
rollout horizon, or goal cost.

Across three independently trained recurrent checkpoints and matched evaluation
roots, adaptive depth improves the best fixed-depth mean on Reacher, Cube
Double, Cube Triple, and PushT $H=10$, and matches it on Cube Single, while
averaging only $K=1.00$--$1.26$. More importantly, the experiments reveal two
principles behind the gain. First, compute is \emph{interaction aligned}: the
probability of thinking beyond $K=1$ rises sharply at PushT contact onset and
during sustained Cube Triple contact. Second, the useful computation is
\emph{decision facing}: it need not make object state globally more decodable.
Instead, exact-batch traces show that small latent corrections reorder CEM
costs, alter elite membership, and change the selected action near a stochastic
search boundary. The sparsity and search-seed dependence of outcome flips are
evidence of this boundary mechanism, not evidence that internal depth is
irrelevant.

Our contributions are:
\begin{itemize}
  \setlength{\itemsep}{1pt}
  \setlength{\parsep}{0pt}
  \setlength{\parskip}{0pt}

  \item We identify \textbf{internal transition depth} as a distinct test-time
  scaling axis and formulate its allocation as a value-of-computation problem,
  where each recurrent update is a metalevel action valued through the planner's
  eventual decision.
  \item We introduce \textbf{DeepJEPA}, which turns each latent transition into
  a budgeted computation process and improves or matches the complete
  fixed-depth frontier across five visual-control settings while averaging
  $\bar K=1.00$--$1.26$.
  \item We characterize \textbf{interaction-aligned, decision-boundary
  refinement} through an elite-margin stability result that identifies when
  internal updates cannot change CEM. Experiments show that useful computation
  concentrates around physical interaction and acts through planner rankings
  rather than uniform gains in state decodability.
\end{itemize}

\section{Related Work}

\textbf{World models and latent planning.}
Learned dynamics support planning from Dyna, visual foresight, and probabilistic
MPC to World Models, Dreamer, model-based policy optimization, MuZero, and
TD-MPC~\citep{dyna,finn2017_visualforesight,chua2018_pets,ha2018world,
hafner2020_dreamer,janner2019_mbpo,schrittwieser2019_muzero,
hansen2022temporal,hafner2023_dreamerv3,hansen2023td}. I-JEPA and V-JEPA replace
pixel reconstruction with feature prediction~\citep{assran2023self,
bardes2024revisiting}. V-JEPA~2-AC and LeWM extend this idea to action-conditioned
prediction and latent planning~\citep{assran2025v,maes2026leworldmodel}, while
related planners study frozen features, offline dynamics, human video,
invariance, temporal straightening, and hierarchy~\citep{dinowm,
sobal2025learning,dexwm,toso2026learning,wang2026temporal_straightening,
zhang2026hierarchical}. These methods scale temporal rollout, optimization, or
search while normally fixing computation inside one transition. DeepJEPA
instead allocates recurrent depth before each imagined transition is accepted.

\textbf{Robot policies and world action models.}
Generalist robot policies connect visual and language representations to action
through large-scale sequence, diffusion, and open-source policy models
\citep{Brohan2022RT1RT,Brohan2023RT2VM,ghosh2024octo,kim2024openvla,
black2025pi0,ye2026datapyramidembodiedmanipulation}. A newer line couples action generation to predictive video or
explicit world-model structure~\citep{kim2026cosmos,sun2026vlajepa,
cen2025worldvla,ye2026dreamzero,yuan2026fastwam,li2026lightwam}.
DeepJEPA addresses an orthogonal question: once a world model is used inside a
planner, how much internal computation should each imagined transition receive?

\textbf{Adaptive test-time computation.}
Inference-time scaling varies samples, search, or verification effort
\citep{wang2022self,yao2023tree,lightman2024let,brown2024large}, with
input-dependent benefits that motivate compute-optimal allocation
\citep{snell2025scaling,xu2026pointsequaluncertaintyaware4d}. Test-time training adapts parameters through
self-supervision or reconstruction~\citep{sun2020ttt,wang2021tent,
zhang2022memo,gandelsman2022tttmae,wang2023videottt,zhang2025ot,
zhang2025dpcore,zhang2026adapting}, and recent world models
adapt planning or latent dynamics~\citep{wang2025adawm,gao2025adaworld,
lanier2025redraw}. Adaptive Computation Time, Universal Transformers,
PonderNet, and early exits instead learn input-dependent network depth
\citep{graves2016adaptive,dehghani2018universal,banino2021pondernet,
xin2020deebert}. DeepJEPA shares the halting idea but neither updates parameters
nor changes the outer planning budget. It allocates depth to one transition of
one candidate trajectory, where value is determined by the resulting rank.

\section{DeepJEPA: Scaling Within a Transition}
\label{sec:method}

\subsection{Budgeted recurrent transitions and learned halting}

Let an MPC rollout contain candidate $i$, temporal step $h$, and internal depth
$k$. Conventional scaling changes the number of candidates or rollout steps
while treating each transition as one fixed-cost prediction. DeepJEPA instead
selects $K_{i,h}\in\{1,\ldots,K_{\max}\}$ for every candidate and time pair.
CEM allocates search across trajectories, while DeepJEPA allocates computation
within them. Its stopping rule is trained from marginal prediction improvement
without a control reward, then tested by whether the resulting allocation
changes the planner's decision.

Let $E_\theta$ be the online visual encoder and $E_{\bar\theta}$ the target
encoder. DeepJEPA conditions each imagined transition on the last $L$ latent
states and actions ($L=3$ in our experiments). A context network and shared
recurrent transition cell compute
\begin{align}
  c_t &= C_\theta(z_{t-L+1:t}, a_{t-L+1:t}), \\
  h_{t+1}^{(k)} &= F_\theta\!\left(h_{t+1}^{(k-1)},c_t,
    b_{t+1}^{(k-1)}\right), \\
  \hat z_{t+1}^{(k)} &= P_\theta(h_{t+1}^{(k)}),
  \qquad k=1,\ldots,K_{\max},
\end{align}
where $b^{(k-1)}$ feeds the previous prediction back into the shared cell. The
state is conditioned on the candidate action, and parameters are tied across
$k$. We train through $K_{\max}=4$ and append the selected prediction to the
imagined history before rolling to the next temporal step.

During training the target encoder provides
$z^*_{t+1}=E_{\bar\theta}(o_{t+1})$. Define the detached per-depth error
$e_k=\lVert\hat z_{t+1}^{(k)}-\operatorname{sg}(z^*_{t+1})\rVert_2^2$ and the
relative improvement
\begin{equation}
  r_k = \frac{e_k-e_{k+1}}{\max(e_k,\epsilon)}, \qquad
  y_k = \mathbb{I}[r_k>\tau_{\mathrm{rel}}].
\end{equation}
We use $\tau_{\mathrm{rel}}=5\times10^{-4}$. A linear continue head predicts
$p_k=\sigma(H_\phi(h_{t+1}^{(k)}))$ and is trained with binary cross entropy.
The full predictor loss combines final-depth latent prediction,
stop-gradient intermediate supervision, and the continue loss:
\begin{equation}
  \mathcal L = \mathcal L_{\mathrm{final}} +
  \lambda_{\mathrm{inter}}\mathcal L_{\mathrm{inter}} +
  \beta\mathcal L_{\mathrm{cont}}.
\end{equation}
At test time $z^*_{t+1}$ is unavailable and is not used. Computation halts at
the first depth with $p_k\leq\eta$, or at $K_{\max}$. Thus $\eta$ controls the
success and selected-depth operating point. Complete training and evaluation
settings are provided in App.~\ref{app:protocol}. Fig.~\ref{fig:method}
summarizes how the same recurrent cell supports adaptive inference and how its
continue targets are constructed only during training.

\begin{figure}[t]
  \centering
  \includegraphics[width=\textwidth]{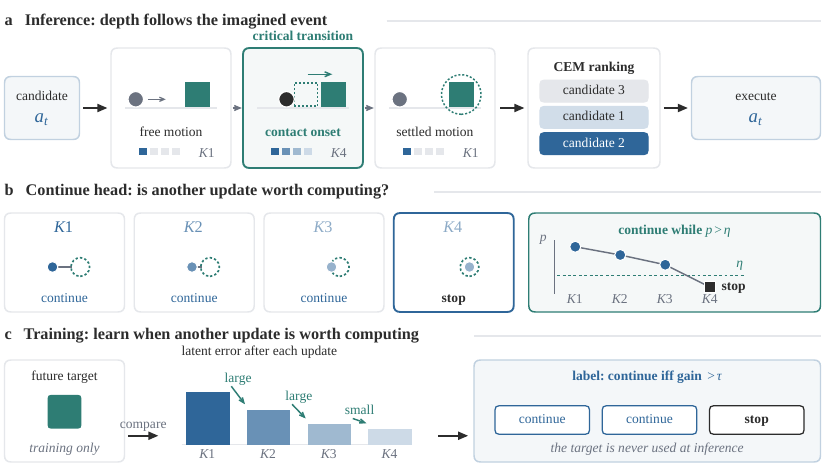}
  \caption{\textbf{DeepJEPA scales world-model computation from within.}
  (a)~At inference each CEM candidate enters a weight-tied transition cell
  that may update its latent prediction up to four times; filled pips give the
  depth selected for that transition. (b)~A continue head decides depth
  independently for every candidate--time pair: the dashed ring is the target
  latent, the dot the current prediction, and the head continues while the
  predicted marginal gain exceeds~$\eta$. (c)~During training, the marginal
  reduction in latent error after each update supplies the continue label.
  Future target latents are used only in training and never at inference.}
  \label{fig:method}
\end{figure}

\subsection{Planner-relative computation and CEM}

Rational metareasoning treats computation as an action whose utility is derived
from its effect on an agent's external decision~\citep{russell1991principles,
hay2014selecting}. We apply this view inside an imagined world-model transition.
Let $S_{i,h}^{(k)}$ denote the information available after $k$ internal updates
for candidate $i$ at rollout step $h$, and let $V(S)$ denote the expected task
utility induced by the fixed outer planner from that information. The net value
of computing one more update is
\begin{equation}
  \operatorname{VOC}_{i,h}^{(k)}
  = \mathbb E\!\left[
      V\!\left(S_{i,h}^{(k+1)}\right)
      -V\!\left(S_{i,h}^{(k)}\right)
      \mid S_{i,h}^{(k)}
    \right]
    -\lambda_{\mathrm{cmp}},
  \label{eq:voc}
\end{equation}
where $\lambda_{\mathrm{cmp}}$ prices an additional recurrent update. A
one-step metalevel rule continues exactly when
$\operatorname{VOC}_{i,h}^{(k)}>0$. With diminishing marginal gains, a common
compute price yields a threshold allocation in which easy transitions stop
early. The full compute-regularized objective, assumptions, and proxy
interpretation are given in App.~\ref{app:theory_details}.

DeepJEPA does not observe the counterfactual task utility in
Eq.~\ref{eq:voc}. Instead, $r_k$ supplies a reward-free local proxy, and the
continue head amortizes whether that proxy exceeds $\tau_{\mathrm{rel}}$.

\textbf{Integration with CEM.}
The selected prediction $\hat z_{t+h+1}^{(i,K_{i,h})}$ is rolled forward and
scored exactly as in LeWM. Candidate distributions, elites, CEM iterations, and
rollout length remain fixed. We report $\bar K$ over all imagined transitions
as selected architectural depth, not proportional wall-clock speedup.

\subsection{Theoretical support: when refinement can change CEM}

The decision value of depth can be made precise through the object CEM
actually consumes, namely the ordering of candidate costs. We analyze one CEM
iteration while holding its sampled action sequences fixed. Let $N$ be the
number of candidates and let $M$, with $1\leq M<N$, be the number retained as
elites. For candidate $i\in\{1,\ldots,N\}$, let $J_i^{(k)}\in\mathbb R$ denote
the goal cost obtained after its $k$th internal update, where lower cost is
better. Write $J_{(m)}^{(k)}$ for the $m$th smallest cost, and let
$\mathcal E_k$ be the indices of the $M$ candidates with smallest costs at
depth $k$. A positive gap between $J_{(M)}^{(k)}$ and
$J_{(M+1)}^{(k)}$ makes this elite set unique.

\begin{table}[t]
\centering
\small
\setlength{\tabcolsep}{9pt}
\caption{\textbf{Adaptive depth improves the fixed-depth frontier on four of
five settings.} Mean success rate in percent over three training/evaluation
seed pairs, 50 episodes each (150 episodes per cell). Single, Double and
Triple are the Cube tasks. LeWM is a separate non-recurrent checkpoint family,
not a depth setting. The last row is in recurrent updates, not percent.
Per-seed standard deviations are in Table~\ref{tab:main_spread}. Higher is
better.}
\label{tab:main_numbers}
\begin{tabular}{@{}lrrrrr@{}}
\toprule
Method & Reacher & Single & Double & Triple & PushT \\
\midrule
LeWM              & 81.3 & 72.0 & 74.7 & 74.0 & 10.4 \\
$K=1$             & 84.0 & 79.3 & 72.0 & 74.0 & 10.7 \\
$K=2$             & 82.0 & 78.7 & 72.0 & 74.0 & 10.7 \\
$K=3$             & 83.3 & 78.0 & 72.7 & 73.3 & 10.4 \\
$K=4$             & 82.7 & 78.0 & 72.0 & 74.0 &  9.8 \\
\midrule
\ours{DeepJEPA}   & \ours{85.3} & \ours{79.3} & \ours{74.0} & \ours{77.3} & \ours{11.6} \\
\midrule
\refarm{Mean depth $\bar{K}$} & \refarm{1.03} & \refarm{1.00} & \refarm{1.26} & \refarm{1.22} & \refarm{1.02} \\
\bottomrule
\end{tabular}
\end{table}

We compare the elite set before and after one additional recurrent update.
Define the elite margin $\gamma_k$ and the cost correction
$\delta_i^{(k)}$ by
\begin{equation}
  \gamma_k = J_{(M+1)}^{(k)}-J_{(M)}^{(k)},
  \qquad
  \delta_i^{(k)}=J_i^{(k+1)}-J_i^{(k)}.
  \label{eq:elite_margin}
\end{equation}
The quantity $\gamma_k$ is the boundary separating the worst current elite
from the best current nonelite. The vector
$\delta^{(k)}=(\delta_1^{(k)},\ldots,\delta_N^{(k)})$ collects the changes in
candidate cost caused by the next update, and
$\lVert\delta^{(k)}\rVert_\infty:=
\max_{1\leq i\leq N}|\delta_i^{(k)}|$ is its largest absolute correction. For
an adaptive update, a candidate that halts can be represented by
$\delta_i^{(k)}=0$, so the same analysis applies when only part of the batch is
refined.

\begin{proposition}[Elite set stability]
\label{prop:elite_stability}
Assume $\gamma_k>0$. If
\begin{equation}
  2\lVert\delta^{(k)}\rVert_\infty < \gamma_k,
  \label{eq:elite_stability_condition}
\end{equation}
then the $M$-candidate CEM elite set is unchanged by update $k+1$. For the
standard unweighted elite update, the next proposal mean and covariance are
also unchanged. We provide a complete proof in App.~\ref{app:elite_margin}.
\end{proposition}

The proof is a direct ordering argument. For any current elite $i$ and current
nonelite $j$, their original cost gap is at least $\gamma_k$. After refinement,
their gap is at least
$\gamma_k-2\lVert\delta^{(k)}\rVert_\infty$, which remains positive under
Eq.~\ref{eq:elite_stability_condition}. No elite and nonelite pair can exchange
order. The selected action sequences are therefore identical, and the standard
unweighted CEM update computes the same proposal statistics from them. The
formal proof in App.~\ref{app:elite_margin} states each comparison explicitly.

This proposition is a sufficient stability guarantee, not a claim that deeper
prediction always improves control. It identifies a regime in which another
update has exactly zero value to the next CEM proposal even when the latent and
cost both change. Conversely, when the margin contracts, a small
candidate-specific correction may cross the elite boundary and alter the next
action distribution. This explains why DeepJEPA must allocate depth rather than
apply it uniformly. It also yields three testable predictions: outcome changes
should be sparse, depend on the sampled search population, and need not coincide
with improved global state decodability. Section~\ref{sec:experiments} tests
these consequences directly.

\section{Experiments}
\label{sec:experiments}

\textbf{Protocol.}
We evaluate Reacher, visual Cube Single, Double, and Triple, and PushT
\citep{chi2024pusht}. Every setting uses three training and evaluation seed
pairs and 50 episodes per pair. Fixed $K=1,2,3,4$ and learned halting share a
recurrent checkpoint family, roots, and CEM seeds. LeWM is a separate
non-recurrent baseline. We report the best mean-success threshold from the full
sweep in Table~\ref{tab:thresholds}. This characterizes the achievable frontier
rather than held-out threshold selection. Complete planner, seed, episode, and
probe settings are in App.~\ref{app:protocol}.

\noindent\textbf{Q1. Is uniformly deeper prediction a reliable scaling rule?}
Table~\ref{tab:main_numbers} gives the complete comparison. Uniform depth is not
a reliable scaling rule. Within the same recurrent checkpoint family,
aggregate fixed $K=4$
never exceeds fixed $K=1$: it is lower on Reacher, Cube Single, and PushT, and
tied on Cube Double and Cube Triple. Intermediate depths are also
non-monotonic. For example, Reacher moves from $84.0\%$ at $K=1$ to $82.0\%$ at
$K=2$, $83.3\%$ at $K=3$, and $82.7\%$ at $K=4$. PushT likewise peaks at
$10.7\%$ for $K\in\{1,2\}$ and declines to $9.8\%$ at $K=4$. The added
computation is real, but its utility varies with the imagined transition and
planner state. These comparisons hold the recurrent checkpoint, episode roots,
candidate budget, and CEM seeds fixed. The observed reversals therefore cannot
be explained by a larger planner or a different predictor family. They show
that deeper updates can change candidate costs without consistently improving
the ranking that control ultimately consumes. If depth had uniform value, the
fixed sweep would exhibit a stable ordering across $K$. Its failure to do so
motivates allocating computation at the transition level rather than choosing
one globally deeper model.

\takeawaybox{1}{Uniform Depth Is Not a Scaling Law}{Transition depth is a real
test-time compute axis, but its value is state dependent. Applying recurrent
updates everywhere can waste compute and erase useful rankings.}

\begin{figure}[t]
  \centering
  \includegraphics[width=\textwidth]{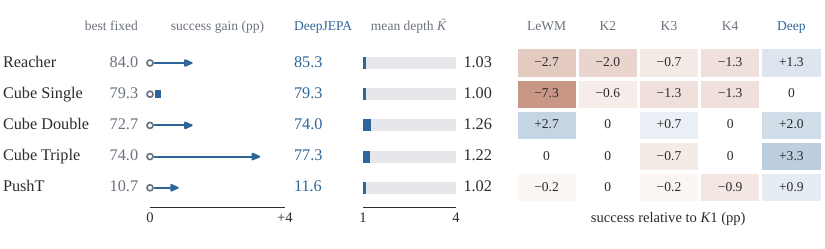}
  \caption{\textbf{DeepJEPA improves the fixed-depth frontier while staying at
  the shallow endpoint.} Left: the move from the best fixed depth (hollow dot)
  to DeepJEPA (arrowhead) in percentage points; the two absolute success rates
  are printed either side of the track, and the stub on Cube Single is a
  measured zero, not a missing bar. Centre: mean selected depth $\bar{K}$ as a
  bar inside the full $K=1$ to $K=4$ budget the model could have spent. Right:
  every arm relative to $K=1$, so a reader can check that no fixed depth beats
  the adaptive one; blue is above $K=1$, brown below. Means over three
  training/evaluation seed pairs, 150 episodes per cell; per-seed spreads are
  in Table~\ref{tab:main_spread}. Higher is better.}
  \label{fig:main}
\end{figure}

\noindent\textbf{Q2. Can adaptive depth improve the fixed-depth frontier?}
Fig.~\ref{fig:main} visualizes the full success and selected-depth frontier.
DeepJEPA changes the shape of this frontier. It reaches $85.3\%$ on Reacher,
$74.0\%$ on Cube Double, $77.3\%$ on Cube Triple, and $11.6\%$ on PushT,
above the best fixed recurrent mean in each setting. It matches Cube Single at
$79.3\%$. The corresponding mean depths are $1.03$, $1.26$, $1.22$, $1.02$,
and $1.00$. DeepJEPA also exceeds the non-recurrent LeWM mean on four settings.
Cube Double separates the comparisons: adaptive depth improves its recurrent
family from $72.7\%$ to $74.0\%$, while LeWM obtains $74.7\%$. Thus DeepJEPA
does not approximate one globally deeper network. It stays near the shallow
endpoint and selectively invokes additional updates. Cube Double further
isolates allocation within the same recurrent family from a change of predictor.
The gain-to-depth pairs sharpen this result. Reacher gains $1.3$ points over its
best fixed mean at $\bar K=1.03$, Cube Triple gains $3.3$ points at
$\bar K=1.22$, and PushT gains $0.9$ points at $\bar K=1.02$. Cube Single
selects $\bar K=1.00$ and exactly matches the best fixed recurrent result. The
halting policy can therefore recover the shallow solution when refinement has
little measured value, while concentrating a small additional budget on tasks
and transitions where the fixed frontier leaves room to improve.

\takeawaybox{2}{Allocation Beats Uniform Depth}{DeepJEPA improves or matches the
complete fixed-depth frontier while remaining near the shallow endpoint. The
important quantity is where the extra recurrent steps are spent.}

\begin{wraptable}[17]{r}{0.62\textwidth}
\centering

\caption{\textbf{Threshold controls allocation, not average depth.} Success
and selected depth across $\eta$. Bold marks the main operating point.}
\label{tab:thresholds}
\scriptsize
\begin{tabular*}{\linewidth}{@{\extracolsep{\fill}}lrrrrr@{}}
\toprule
Task & $\eta=.30$ & $.45$ & $.50$ & $.60$ & $.70$ \\
\midrule
\multicolumn{6}{@{}l}{\refarm{Mean success (\%)}} \\
Reacher & 82.7 & \ours{85.3} & 84.0 & 84.7 & 82.7 \\
Single  & 78.0 & 77.3 & 78.0 & 77.3 & \ours{79.3} \\
Double  & 73.3 & \ours{74.0} & 73.3 & 72.0 & 72.7 \\
Triple  & \ours{77.3} & 73.3 & 74.7 & 73.3 & 72.7 \\
\midrule
\multicolumn{6}{@{}l}{\refarm{Mean selected depth $\bar{K}$}} \\
Reacher & 1.09 & 1.03 & 1.02 & 1.01 & 1.00 \\
Single  & 1.09 & 1.04 & 1.03 & 1.01 & 1.00 \\
Double  & 1.78 & 1.26 & 1.07 & 1.00 & 1.00 \\
Triple  & 1.22 & 1.13 & 1.05 & 1.00 & 1.00 \\
\bottomrule
\end{tabular*}
\end{wraptable}
\noindent\textbf{Threshold sweep.}
The continue threshold $\eta$ changes both how often DeepJEPA reasons beyond
$K=1$ and which transitions receive those updates. Table~\ref{tab:thresholds}
shows the complete sweep for the four paired-outcome tasks. The best operating
point is task dependent: Reacher peaks at $\eta=.45$ with $\bar K=1.03$, Cube
Double at $.45$ with $\bar K=1.26$, and Cube Triple at $.30$ with
$\bar K=1.22$. The sweep also reveals that average depth alone does not explain
performance. On Cube Double, reducing $\bar K$ from $1.78$ to $1.26$ improves
success. On Reacher, a small change from $1.09$ to $1.03$ raises success by
$2.6$ points. The halting rule therefore controls an allocation frontier, not a
scalar ``more compute is better'' curve.

\begin{figure}[t]
  \centering
  \begin{subfigure}[t]{0.50\textwidth}\centering
    \includegraphics[width=\textwidth]{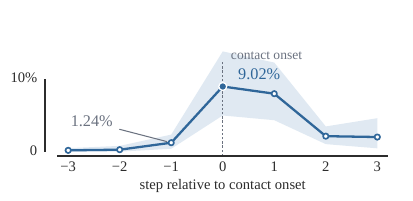}
    \caption{PushT: refinement spikes at contact}
  \end{subfigure}
  \hfill
  \begin{subfigure}[t]{0.46\textwidth}\centering
    \includegraphics[width=\textwidth]{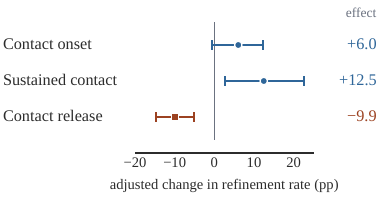}
    \caption{Cube Triple: depth follows contact phase}
  \end{subfigure}
  \caption{\textbf{Adaptive computation is interaction aligned.} (a)~Fraction
  of PushT transitions refined beyond $K=1$, aligned on contact onset; the
  band is the 95\% cluster-bootstrap interval. (b)~Phase effects on refinement
  probability for Cube Triple, adjusted for rollout-step fixed effects,
  distance and action norm; dots are point estimates, whiskers 95\%
  cluster-bootstrap intervals, and squares mark negative effects. Blue is more
  refinement, brown less. 100 roots per task, one checkpoint per task.}
  \label{fig:allocation}
\end{figure}

\noindent\textbf{Q3. Where does the model think deeper?}
We next analyze 100 roots per task using 32 candidates sampled from the final
CEM iteration (32,000 PushT and 16,000 Cube Triple transitions).
Fig.~\ref{fig:allocation} shows a sharp PushT contact-onset response and a
phase-dependent Cube Triple result after controlling for rollout step, distance,
and action norm. Cube Triple sustained contact increases the probability of
refining beyond $K=1$ by $12.5$ points (95\% CI $[2.8,22.6]$), while contact
release reduces it by $9.9$ points ($[-14.6,-5.1]$). Contact onset is positive
with a wider interval: $+6.0$ points ($[-0.5,12.2]$). On PushT, the same pattern
appears as a sharp event-aligned spike: transitions refined beyond $K=1$ rise
from $1.24\%$ immediately before contact to $9.02\%$ at onset
(95\% CI $[5.01,13.88]$) and remain at $8.03\%$ one step later.
The agreement across these physically different settings is important. The
continue head receives only marginal latent-improvement supervision and no
contact labels, yet it assigns depth when local dynamics become coupled. This
\emph{interaction-aligned depth} follows physical phase rather than rollout
position.

\takeawaybox{3}{Depth Follows Physical Interaction}{DeepJEPA concentrates
refinement at contact onset and sustained interaction, then stops when another
update is unlikely to matter.}

\begin{figure}[t]
  \centering
  \begin{subfigure}[t]{0.47\textwidth}\centering
    \includegraphics[width=\textwidth]{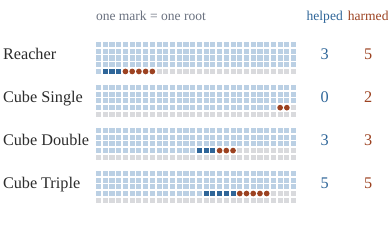}
    \caption{Depth flips few outcomes}
  \end{subfigure}
  \hfill
  \begin{subfigure}[t]{0.25\textwidth}\centering
    \includegraphics[width=\textwidth]{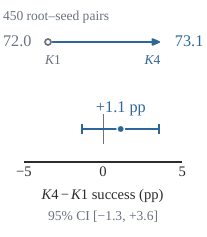}
    \caption{Net effect crosses zero}
  \end{subfigure}
  \hfill
  \begin{subfigure}[t]{0.22\textwidth}\centering
    \includegraphics[width=\textwidth]{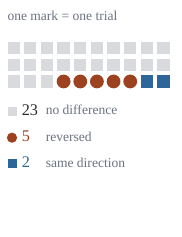}
    \caption{Search changes it}
  \end{subfigure}
  \caption{\textbf{Additional depth acts at CEM decision boundaries.}
  (a)~One mark per root, 150 roots per task, ordered so the flips form a
  seam: pale blue squares succeed at both $K=1$ and $K=4$, pale grey squares
  fail at both, blue squares are helped by depth and brown discs are harmed.
  93--99\% of roots never change outcome. (b)~Across 450 Cube Triple root--search-seed pairs the net effect
  is $+1.1$ points and its 95\% interval crosses zero (paired exact
  $p=0.424$). (c)~Re-sampling the search on the ten reference-discordant roots
  changes the flip direction in five of thirty trials. Together with
  exact-batch replay this is the signature of a local ranking mechanism: depth
  matters when a small cost correction can change the elite set.}
  \label{fig:mechanism}
\end{figure}

\noindent\textbf{Q4. How does refinement change a planning decision?}
The paired split in Fig.~\ref{fig:mechanism} first localizes the effect. Across
Reacher and the three Cube tasks, $93$--$99\%$ of roots have the same binary
outcome under $K=1$ and $K=4$. On Cube Triple, 106 roots succeed at both depths
and 34 fail at both. Only five are helped and five harmed. Uniform depth is
inefficient precisely because most roots are far from the planner boundary.
Exact-batch replay reproduces 297 of 300 individual outcome labels and nine of
ten discordant categories, ruling out a logging mismatch. With three
alternative search seeds per Cube Triple
root (450 root--seed pairs), $K=1$ obtains $72.0\%$ and $K=4$ obtains $73.1\%$.
The difference is $+1.1$ points with
cluster-bootstrap CI $[-1.3,+3.6]$ and paired exact $p=0.424$. Among 30
alternative trials on the ten reference-discordant roots, 23 show no difference,
five reverse direction, and two preserve it. The outcome changes only when both
model depth and sampled candidates place the optimizer near a ranking boundary.
The dependence on search seed does not erase the mechanism. It localizes the
mechanism to candidate populations whose cost ordering lies close enough to an
elite boundary for a depth-induced correction to matter.

The exact traces expose the causal interface. Holding the CEM candidate batch
fixed, deeper latent updates produce small, candidate-specific goal-cost
corrections. Those corrections can move trajectories into or out of the elite
set. The updated Gaussian then proposes a different action. CEM therefore acts
as an amplifier that converts a local representation change into a closed-loop
decision. Global average prediction error can remain almost unchanged while the
rank of the few candidates that determine control changes materially.

\takeawaybox{4}{Refinement Matters at Ranking Boundaries}{Internal depth has
decision value when candidate-specific cost corrections cross a CEM elite
boundary and change the action distribution.}

\begin{figure}[t]
  \centering
  \begin{subfigure}[t]{0.52\textwidth}\centering
    \includegraphics[width=\textwidth]{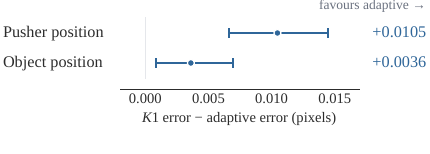}
    \caption{PushT: negligible change}
  \end{subfigure}
  \hfill
  \begin{subfigure}[t]{0.44\textwidth}\centering
    \includegraphics[width=\textwidth]{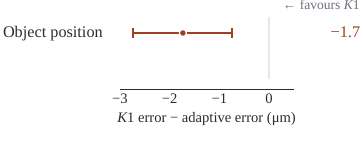}
    \caption{Cube Triple: opposite sign}
  \end{subfigure}
  \caption{\textbf{Global object-state decodability is not the mechanism.}
  Paired MLP-probe error differences between fixed $K=1$ and adaptive
  DeepJEPA on 10{,}000 held-out transitions per task. Dots are paired means,
  whiskers 95\% intervals. Positive favours adaptive depth. PushT effects stay
  below $0.011$ pixels; Cube Triple object position moves $1.7\,\mu$m in the
  opposite direction. This null rules out a uniform state-semantics
  explanation and supports the local CEM-ranking mechanism.}
  \label{fig:probe_delta}
\end{figure}

\noindent\textbf{Q5. Is the gain explained by globally better state features?}
Fig.~\ref{fig:probe_delta} directly tests the tempting alternative explanation
that extra
depth simply produces a more accurate object-centric latent. On 10,000 held-out
PushT transitions, adaptive refinement changes MLP pusher- and object-position
errors relative to the same recurrent model at $K=1$ by only $0.0105$ and
$0.0036$ pixels. On Cube Triple, adaptive object-position error is worse by
$1.7\,\mu$m. Cross-model comparisons can still favor the recurrent predictor
over LeWM, but within DeepJEPA the selected additional updates do not create a
useful global decodability gain. This rules out a uniform state-semantics account
and agrees with the exact CEM traces, where sparse candidate-specific updates
matter through relative cost ordering.

This negative result is mechanistically informative. A uniform semantics
account predicts a broad probe gain, which we do not observe. Instead, a small
number of candidate-specific updates matter through relative cost rather than
through a globally better state representation.

\takeawaybox{5}{Decision Value Is Not Global Decodability}{DeepJEPA improves
planner-relevant rankings without uniformly improving object-state probes. Its
useful computation is selective and rank aware.}

\FloatBarrier

\section{What Changes About World-Model Scaling?}

\textbf{Scaling from within is metalevel resource allocation.}
Rollout horizon and candidate count describe how far and broadly a model
imagines, but not how much computation it spends before accepting one state
change. DeepJEPA separates these axes. Search chooses which futures to consider,
while the transition model chooses which changes deserve deeper processing.
Contact onset and sustained interaction emerge as high-compute regimes even
though the continue head receives no contact labels. Adaptive depth can thus act
as an internal event detector that follows changes in scene dynamics rather than
a fixed network schedule.

\textbf{Prediction quality is planner relative.}
CEM consumes ranks and elite sets, not average feature decodability.
Proposition~\ref{prop:elite_stability} identifies the boundary: refinement
cannot change the CEM update when its largest cost correction is less than half
the elite margin, while a smaller correction can matter when that margin
contracts. Future continue targets can therefore estimate rank margin,
elite-set stability, or action sensitivity directly. This is the purpose of
our analysis beyond the empirical gains. It turns adaptive depth into a general
principle for allocating world-model computation at the interface between
prediction and decision. Full experimental scope, threshold-selection caveats,
and compute limitations are documented in App.~\ref{app:limitations}.

\section{Conclusion}

DeepJEPA scales world models from within each imagined transition. Across five
visual-control settings, it improves or matches the fixed-depth frontier while
remaining near $K=1$ on average. Its deeper updates concentrate around physical
interaction and reshape CEM rankings near decision boundaries rather than
uniformly improving state decodability. Together, the results separate two
notions of world-model scale: how many futures a planner imagines and how much
computation the model spends before accepting each imagined change. The elite
stability result provides a concrete boundary between useful and inert internal
computation. When candidate cost corrections cannot cross the elite margin,
another update leaves the CEM proposal unchanged. When that margin contracts,
even a small candidate-specific correction can redirect the search. This link
between latent refinement and planner sensitivity suggests a broader design
principle for predictive control. Future world models should learn compute
policies at the prediction-to-decision interface, potentially using rank
stability or action sensitivity as more direct training signals. The design
principle is simple: \emph{think before you roll}, spending depth where a better
prediction can change the decision.

\section*{Acknowledgements}
We thank Quentin Le Lidec, Lucas Maes, and Randall Balestriero for helpful
discussions and valuable feedback.

\bibliography{iclr2027_conference}
\bibliographystyle{apalike}

\appendix

\section{Experimental Protocol}
\label{app:protocol}

Table~\ref{tab:protocol} summarizes the main evaluation. Reacher and Cube use
300 CEM candidates, 30 iterations, 30 elites, and 100-step episodes. Cube goals
are sampled 25 dataset steps after the initial observation. All fixed-depth and
learned-depth evaluations within a recurrent checkpoint use the same episode
roots and CEM seeds. Training seeds are 3072, 3073, and 3074, paired with
evaluation seeds 42, 43, and 44. Each pair contributes 50 episodes, giving 150
episodes per setting. We evaluate
$\eta\in\{0.30,0.45,0.50,0.60,0.70\}$ and use the highest mean-success point
for the main comparison. The one-step probes use a history of 3, frame skip 5,
30,000 training frames, 6,000 validation frames, and 10,000 evaluation
transitions. Contact analyses use 100 roots and one checkpoint per task.

\begin{table}[h]
\centering
\small
\setlength{\tabcolsep}{8pt}
\caption{\textbf{Main evaluation protocol.} Each row uses three
training/evaluation seed pairs and 50 episodes per pair. Cube goals are
sampled 25 dataset steps after the initial observation. All fixed-depth and
learned-depth evaluations within a recurrent checkpoint share episode roots
and CEM seeds. PushT uses the planner configuration shipped with the
benchmark, identical across all methods, marked ---.}
\label{tab:protocol}
\begin{tabular}{@{}lrrrrrr@{}}
\toprule
Setting & Horizon & Action block & CEM samples & Iterations & Elites & $\eta$ \\
\midrule
Reacher     &  5 & 5   & 300 & 30 & 30 & 0.45 \\
Cube Single &  5 & 5   & 300 & 30 & 30 & 0.70 \\
Cube Double &  5 & 5   & 300 & 30 & 30 & 0.45 \\
Cube Triple &  5 & 5   & 300 & 30 & 30 & 0.30 \\
PushT       & 10 & --- & --- & --- & --- & 0.60 \\
\bottomrule
\end{tabular}
\end{table}

\begin{table}[h]
\centering
\small
\setlength{\tabcolsep}{9pt}
\caption{\textbf{Seed spread for every entry of Table~\ref{tab:main_numbers}.}
Standard deviation of success rate in percentage points across the same three
training/evaluation seed pairs. Rows and columns match
Table~\ref{tab:main_numbers} exactly.}
\label{tab:main_spread}
\begin{tabular}{@{}lrrrrr@{}}
\toprule
Method & Reacher & Single & Double & Triple & PushT \\
\midrule
LeWM              & 4.2 & 12.0 & 7.6 & 8.0 & 2.3 \\
$K=1$             & 5.3 &  8.1 & 3.5 & 4.0 & 1.2 \\
$K=2$             & 2.0 &  7.6 & 5.3 & 0.0 & 2.4 \\
$K=3$             & 1.2 &  8.7 & 5.0 & 5.0 & 2.0 \\
$K=4$             & 1.2 &  8.7 & 5.3 & 6.0 & 1.7 \\
\midrule
\ours{DeepJEPA}   & \ours{4.2} & \ours{8.1} & \ours{3.5} & \ours{7.6} & \ours{2.5} \\
\bottomrule
\end{tabular}
\end{table}

\section{Decision-Theoretic Details}
\label{app:theory_details}

The value-of-computation view in Eq.~\ref{eq:voc} can be written as a global
compute-regularized planning objective. Let $K_{i,h}$ be the selected recurrent
depth for candidate $i$ at imagined step $h$, let $U(a_t)$ be the downstream
utility of the action returned by the fixed outer planner, and let
$\lambda_{\mathrm{cmp}}\geq0$ price one recurrent update. The allocation problem
is
\begin{equation}
  \max_{\{K_{i,h}\}}
  \mathbb E[U(a_t)]
  -\lambda_{\mathrm{cmp}}\sum_{i,h}(K_{i,h}-1).
  \label{eq:compute_regularized_planner}
\end{equation}
The first update is required, so the second term charges only computation beyond
$K=1$. This formulation holds the candidate distribution, rollout horizon,
elite count, and CEM iterations fixed. It therefore isolates internal transition
depth from outward search budget.

The one-step rule in Eq.~\ref{eq:voc} follows by comparing the conditional
expected utility gain from update $k+1$ with its price. If expected marginal
gains are nonincreasing with depth, once the value falls below the common price
it cannot become worthwhile at a later depth. The optimal one-step policy then
has threshold form. This argument motivates input-dependent stopping but does
not establish global optimality for a finite stochastic CEM run.

DeepJEPA cannot observe counterfactual task utility for every imagined update.
The relative latent-error reduction $r_k$ is therefore a reward-free local
proxy, and the continue head amortizes whether it exceeds
$\tau_{\mathrm{rel}}$. Consequently $p_k$ is not claimed to equal the true
value of computation. The empirical questions are whether this proxy allocates
depth to decision-critical transitions, whether those updates alter candidate
rankings, and whether adaptive allocation improves the fixed-depth frontier.

\section{Proof of Elite-Set Stability}
\label{app:elite_margin}

We restate the notation needed by Proposition~\ref{prop:elite_stability}. Let
$\mathcal I=\{1,\ldots,N\}$ index a fixed batch of $N$ candidate action
sequences, and let $1\leq M<N$ be the CEM elite count. Candidate $i$ has action
sequence $u_i$ and cost $J_i^{(k)}$ after $k$ recurrent transition updates. Let
$\pi_k$ be a permutation that sorts these costs in nondecreasing order,
\begin{equation}
  J_{\pi_k(1)}^{(k)}\leq \cdots \leq J_{\pi_k(N)}^{(k)},
  \qquad J_{(m)}^{(k)}:=J_{\pi_k(m)}^{(k)}.
\end{equation}
The elite index set is
\begin{equation}
  \mathcal E_k:=\{\pi_k(1),\ldots,\pi_k(M)\},
\end{equation}
and the boundary margin is
\begin{equation}
  \gamma_k:=J_{(M+1)}^{(k)}-J_{(M)}^{(k)}.
\end{equation}
The assumption $\gamma_k>0$ rules out a tie at the elite boundary. One
additional recurrent update changes the candidate costs but not the fixed action
sequences. Its correction vector $\delta^{(k)}\in\mathbb R^N$ has entries
\begin{equation}
  \delta_i^{(k)}:=J_i^{(k+1)}-J_i^{(k)},
  \qquad
  \lVert\delta^{(k)}\rVert_\infty
  :=\max_{i\in\mathcal I}|\delta_i^{(k)}|.
\end{equation}

\begin{proof}[Proof of Proposition~\ref{prop:elite_stability}]
Fix any previous elite candidate $e\in\mathcal E_k$ and any previous non-elite
candidate $n\in\mathcal I\setminus\mathcal E_k$. By the definitions of the
order statistics and the elite set,
\begin{equation}
  J_e^{(k)}\leq J_{(M)}^{(k)},
  \qquad
  J_n^{(k)}\geq J_{(M+1)}^{(k)}.
\end{equation}
Subtracting the first inequality from the second shows that every cross-boundary
pair is separated by at least the elite margin:
\begin{equation}
  J_n^{(k)}-J_e^{(k)}
  \geq J_{(M+1)}^{(k)}-J_{(M)}^{(k)}
  =\gamma_k.
  \label{eq:appendix_pair_margin}
\end{equation}
After one additional recurrent update, the separation of the same pair is
\begin{align}
  J_n^{(k+1)}-J_e^{(k+1)}
  &= J_n^{(k)}-J_e^{(k)}+\delta_n^{(k)}-\delta_e^{(k)} \\
  &\geq \gamma_k+\delta_n^{(k)}-\delta_e^{(k)} \\
  &\geq \gamma_k-|\delta_n^{(k)}|-|\delta_e^{(k)}| \\
  &\geq \gamma_k-2\lVert\delta^{(k)}\rVert_\infty \\
  &>0,
  \label{eq:appendix_positive_gap}
\end{align}
where the first inequality uses Eq.~\ref{eq:appendix_pair_margin}, the second
uses $x\geq-|x|$, the third uses the definition of the infinity norm, and the
last uses the condition in Eq.~\ref{eq:elite_stability_condition}. Hence every
member of $\mathcal E_k$ remains strictly cheaper than every member of
$\mathcal I\setminus\mathcal E_k$ after the update. Since $\mathcal E_k$ has
exactly $M$ elements, these candidates are precisely the $M$ smallest-cost
candidates at depth $k+1$. Therefore
\begin{equation}
  \mathcal E_{k+1}=\mathcal E_k.
  \label{eq:appendix_elite_identity}
\end{equation}

It remains to connect elite-set identity to the next CEM proposal. For the
standard unweighted update on the fixed action-sequence batch, the empirical
proposal parameters computed from an elite set $\mathcal E$ are
\begin{align}
  \mu(\mathcal E)
  &=\frac{1}{M}\sum_{i\in\mathcal E}u_i, \\
  \Sigma(\mathcal E)
  &=\frac{1}{M}\sum_{i\in\mathcal E}
    \bigl(u_i-\mu(\mathcal E)\bigr)
    \bigl(u_i-\mu(\mathcal E)\bigr)^\top.
\end{align}
The action sequences $u_i$ are fixed, and
Eq.~\ref{eq:appendix_elite_identity} gives identical summation indices before
and after refinement. Consequently both the proposal mean and covariance are
unchanged. This proves the proposition.
\end{proof}

\section{Scope and Limitations}
\label{app:limitations}

The main evaluation contains three training and evaluation seed pairs and 150
episodes per setting, with uncertainty reported across evaluation seeds.
Per-task halting thresholds are selected from the measured sweep rather than a
held-out validation set. The resulting values characterize the attainable
success and compute frontier, not a deployment-time threshold-selection rule.
The contact and probe studies use one checkpoint per task, so their uncertainty
quantifies sampled roots or transitions rather than training variation.

Mean selected depth is an architectural compute measure. It does not imply a
proportional wall-clock reduction under batched execution, where a small number
of continuing candidates can retain the cost of another synchronized update.
PushT success also remains low in absolute terms. Finally, helped and harmed
episode labels depend on both the learned model and the sampled CEM population.
They are not permanent semantic classes of initial states. These boundaries
limit quantitative generalization while leaving the central empirical pattern
intact: uniform depth is non-monotonic, adaptive depth is sparse and interaction
aligned, and its control effect enters through planner rankings.



\end{document}